\documentclass{article}

    \usepackage[preprint]{neurips_2020}

\usepackage[utf8]{inputenc} % allow utf-8 input
\usepackage[T1]{fontenc}    % use 8-bit T1 fonts
\usepackage{hyperref}       % hyperlinks
\usepackage{url}            % simple URL typesetting
\usepackage{booktabs}       % professional-quality tables
\usepackage{amsfonts}       % blackboard math symbols
\usepackage{nicefrac}       % compact symbols for 1/2, etc.
\usepackage{microtype}      % microtypography
\usepackage{amsmath}
\usepackage{amsthm}
\usepackage{graphicx}
\usepackage{subfigure}

\newtheorem{theorem}{Theorem}
\newtheorem{corollary}{Corollary}
\newtheorem{lemma}{Lemma}
\newtheorem{remark}{Remark}
\DeclareMathOperator*{\argmax}{argmax}

\title{Learning Unstable Dynamical Systems with Time-Weighted Logarithmic Loss}

\author{%
  Kamil Nar \\
  University of California, Berkeley\\
  \includegraphics[scale=0.3]{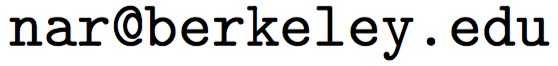}\\
   \And
   Yuan Xue \\
   Google Inc. \\
   \includegraphics[scale=0.3]{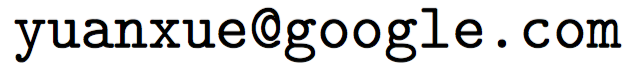}\\
   \And
   Andrew M.~Dai \\
   Google Inc. \\
   \includegraphics[scale=0.3]{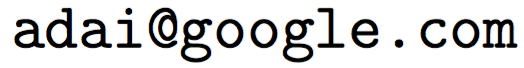}\\
}

\begin{document}

\maketitle

\begin{abstract}%
  When training the parameters of a linear dynamical model,
  the gradient descent algorithm is likely to fail to converge if the squared-error loss is used as the training loss function. Restricting the parameter space to a smaller subset and running the gradient descent algorithm within this subset can allow learning stable dynamical systems, but this strategy does not work for unstable systems. In this work, we look into the dynamics of the gradient descent algorithm and pinpoint what causes the difficulty of learning unstable systems. We show that observations taken at different times from the system to be learned influence the dynamics of the gradient descent algorithm in substantially different degrees. We introduce a time-weighted logarithmic loss function to fix this imbalance and demonstrate its effectiveness in learning unstable systems.

\end{abstract}

\section{Introduction}

Systems with memories that evolve over time require the use of a dynamical model for their representation. This model describes how the memory, or the state, of this system changes over time, how its state is affected by inputs to the system, and how it generates observable outputs. System identification corresponds to the task of learning the unknown parameters of this dynamical model from the known inputs and the observed outputs of the system.

Identification of dynamical systems from time-series data is an important problem for various applications, such as model prediction in reinforcement learning \citep{lambert2019low,zhang2016learning}, analysis of medical health records \citep{rubanova2019latent} and  prediction with financial time-series data \citep{tsay2014financial, ganeshapillai2013learning}.  However, the identification problems that arise in these applications pose some theoretical challenges:
\begin{enumerate}
    \item Unless the state of the system is observed with a known noiseless mapping, the identification of the system model is coupled with the state estimation. Consequently, the system identification task is in general a nonconvex problem \citep{moritz2018jmlr}. To circumvent this nonconvexity, the initial state can be assumed to be zero in control settings, and a known input can be used to drive the state of the system \citep{sastry1984model, sastry1989adaptive}.
    However, in medical and financial settings, the initial state of the system is typically not known a priori, and the deviations of the initial state from a nominal value cannot be neglected. Therefore, a joint and nonconvex optimization procedure is unavoidable in these settings to estimate the initial state of the system along with the unknown model parameters \citep{frigola2014variational, duncker2019learning}.
    
    \item For control of a dynamical system in a reinforcement learning task, it is most critical that the unstable\footnote{The term stability refers to bounded-input bounded-output stability. For continuous-time linear time-invariant systems, this corresponds to the condition where the eigenvalues of the state transition matrix have strictly negative real parts.} modes of the system be discovered and stabilized properly. Similarly, financial data and medical health records usually exhibit sudden changes in their pattern, which call for potentially unstable dynamics in their representation and estimation.
    However, the primary tools for nonconvex optimization, namely, the gradient methods, fail to converge and find an accurate model representation for unstable systems \citep{moritz2018jmlr}. 
    
    \item Especially in medical and financial data sets, the data are sampled irregularly; that is, the observations are not periodically sampled. The common heuristic approach to handle this situation is imputing the absent observations by interpolating the observed values of the output \citep{che2018recurrent}. This approach, however, might fail to capture the correct dynamics of the underlying system. An alternative is to use a model that can take account for the evolution of the state of the system during unobserved intervals without requiring periodic observations \citep{chen2018neural}.
\end{enumerate}

In this work, we use the gradient descent algorithm to identify the unknown parameters of a linear dynamical system from its observed outputs. We look into the dynamics of this algorithm and try to pinpoint what causes the inability of the gradient methods to converge when they are used to identify an unstable dynamical system. Similar to the work of \cite{chen2018neural}, our analysis uses a continuous-time model so that it directly applies to irregularly sampled data sets with no need for imputation.

\subsection{Our contributions}

By analyzing the dynamics of the gradient descent algorithm during identification of a linear dynamical system, we achieve the following.

\begin{enumerate}
    \item We obtain an upper bound on the learning rate of the gradient descent algorithm so that it can converge while learning a dynamical system with the squared-error loss. This upper bound explicitly depends on the eigenvalue of the system with the largest real part, and it shows that identifying a system becomes harder as the system becomes unstable. Furthermore, the upper bound on the learning rate shows that the samples taken at different times affect the convergence of the gradient descent algorithm in substantially different degrees.
    
    \item To enable the convergence of the gradient descent algorithm even when learning unstable systems, we introduce a new loss function which balances the influence of the samples taken at different times on the convergence of the algorithm. Then we demonstrate the effectiveness of this loss function while estimating linear dynamical systems. 
\end{enumerate}

Note that the primary question our work addresses is about the use of the gradient descent algorithm while learning a dynamical system: can this algorithm converge at all while learning the parameters of a dynamical system model? This is a different problem than whether a specific algorithm, or a specific model can learn the dynamical system of interest more accurately than the state-of-the-art.

\subsection{Related works}

\cite{moritz2018jmlr} studied the convergence of the gradient descent algorithm while learning linear dynamical systems. They demonstrated the  failure of this algorithm to learn even stable systems, and proposed a projected gradient method that fixed the issue for linear stable systems. Learning an unstable system, however, was considered to be infeasible. In contrast, we retain the standard gradient descent algorithm in this work, and we introduce a new loss function that allows learning even unstable systems with no necessity for projection.

If the state of a linear system is directly accessed, that is, if the output of the system is equal to the state of the system possibly with some additive noise, learning the system parameters can be formulated as an ordinary least squares problem. \cite{aleaddini2018cdc} and \cite{sarkar2019} make this assumption and arrive at a convex optimization problem. By doing so, they avoid the use of gradient descent algorithm, and therefore, they do not suffer from the issues pointed out by \cite{moritz2018jmlr}. However, as mentioned earlier, the assumption of having an access to the true internal state is unrealistic in many application domains, such as, health and finance.

Using variational inference is a common approach to estimate the initial state jointly with the dynamical model parameters in a Bayesian setting \citep{frigola2014variational, archer2015black, krishnan2017, eleftheriadis2017identification, duncker2019learning, gregor2018temporal}. In this approach, a separate model is employed to estimate the initial state from the whole observed trajectory. One of the models that we will consider in this work is a simpler, deterministic counterpart of this approach. 
%yx: minor changes here.
We show that 
convergence issues of the gradient descent algorithm are also valid for this deterministic counterpart of variational inference.
% In other words, the gradient methods may also fail to estimate unstable dynamical systems for the algorithms that claim to learn a dynamical system and involve a separate estimator to learn the initial state of the system.

Neural ordinary differential equations \citep{chen2018neural, rubanova2019latent} use a neural network to represent a continuous-time dynamical system. Since these models are also trained with the gradient descent algorithm,
they also suffer from the stability issues of the gradient descent algorithm while learning the parameters of a dynamical model. Indeed, the training data of all the examples outlined in these works involve trajectories that converge to either a stable equilibrium or a stable limit cycle of the system.

\section{Problem Formulation}

For each $k \in \mathcal K = \{1, \dots, K\}$, let $z_k : [0, \infty) \mapsto \mathbb R^n$ denote a continuous-time process representing the state of a linear time-invariant dynamical system:
\[ \frac{d z_k(t) }{dt} = A z_k(t) \quad \forall t \ge 0, \ \forall k \in \mathcal K, \] 
where $A \in \mathbb R^{n \times n}$ denotes the state transition dynamics of the system. Then the evolution of the process is described by $z_k(t) = e^{At}z_k(0)$ for all $t \ge 0$ for each $k \in \mathcal K$ \citep{callier1991}.
Let $\{x_k(t)\}_{t \in \mathcal T_k}$ be the set of samples obtained from $z_k$ at time instants $t \in \mathcal T_k$ via an observation matrix $ C \in  \mathbb R^{m \times n}$:
\[ x_k(t) =  Cz_k(t) = C e^{At} z_k(0) \quad \forall t \in \mathcal T_k, \ \forall k \in \mathcal K.\]
Define the initial state of the trajectory of $z_k$ as $s_k \in \mathbb R^n$; that is, let $s_k = z_k(0)$ for all $k \in \mathcal K$. We will look for a linear dynamical system model that fits all the trajectories, and we will use the gradient descent algorithm to reveal the difficulty of its convergence.
In particular, our goal is to study whether the gradient
%yx: minor change.
descent algorithm is able to converge to a solution while solving the problem
%yx: do we need to change  z_k to z_k(t) for all the equations?
\begin{subequations}
\label{eqn:model-0}
\begin{align}
    \underset{A, C}{\text{minimize}} & \quad \sum_{k \in \mathcal K} \sum_{t \in \mathcal T_k} \ell \left( x_k(t),  Ce^{At} s_k \right)
    % \\
    % \text{subject to}
    % & \quad \frac{dz_k}{dt} = A z_k \qquad  \ \forall k \in \mathcal K, \ \forall t \ge 0, \\
    % & \quad z_k(0) = s_k \qquad  \ \  \forall k \in \mathcal K,
\end{align}
\end{subequations}
where $\ell$ is a differentiable loss function. We consider two choices for $\ell$ in the following sections: the squared-error loss as it is used both in classical works \citep{aastrom1971system} and in recent works \citep{moritz2018jmlr}, and the time-weighted logarithmic loss introduced in Section \ref{sec:section4}.

% Note that given the initial states $\{s_k\}_{k \in \mathcal K}$, the trajectories for this system can be written as \citep{callier1991}
% \[ z_k(t) = e^{At}s_k \quad \forall t \ge 0, \ \forall k \in \mathcal K, \]
% where $e^{At}$ denotes the matrix exponential:
% \[ e^{At} = I + \sum\nolimits_{j = 1}^\infty \frac{t^j}{j!} A^j.\]

The set of initial states $\{s_k\}_{k \in \mathcal K}$ is left arbitrary in the statement of (\ref{eqn:model-0}); we consider three possible cases for these initial states, and our analysis in the following sections applies to all of these three cases.

\begin{enumerate}
\item Each $s_k$ is known or has a fixed value. In other words, the set $\{s_k\}_{k \in \mathcal K}$ is not updated by the gradient descent algorithm.

\item Each $s_k$ is also a variable, and the gradient descent algorithm optimizes over $\{s_k\}_{k \in \mathcal K}$ as well:
\begin{equation}
\label{eqn:model-1}
    \underset{A, C, \{s_k\}_{k \in \mathcal K}}{\text{minimize}}  \quad \sum_{k \in \mathcal K} \sum_{t \in \mathcal T_k} \ell \left( x_k(t),  Ce^{At}s_k \right) 
    % \\
    % \text{subject to} & \quad \frac{dz_k}{dt} = A z_k \qquad \forall k \in \mathcal K, \ \forall t \ge 0, \\
    % & \quad z_k(0) = s_k \qquad \  \forall k \in \mathcal K.
\end{equation}

\item Each $s_k$ is output of a state estimator:
\[ s_k = g_\phi \left( \{t, x_k(t)\}_{t \in \mathcal T_k} \right) \quad \forall k \in \mathcal K,\]
where $\phi$ is the parameters of this estimator, and the gradient descent algorithm solves the problem
\begin{subequations}
\label{model:third}
\begin{align}
    \underset{A, C, \phi}{\text{minimize}} & \quad \sum_{k \in \mathcal K} \sum_{t \in \mathcal T_k} \ell \left( x_k(t), Ce^{At}s_k \right)
    + \mathcal L\left( \phi \right)
    \\
    \text{subject to}
    % & \quad \frac{dz_k}{dt} = A z_k \hspace{1.14in} \forall k \in \mathcal K, \ \forall t \ge 0, \\
    & \quad s_k = g_\phi\left( \{t, x_k(t) \}_{t \in \mathcal T_k} \right) \quad   \forall k \in \mathcal K,
\end{align}
\end{subequations}
where $\mathcal L$ is an additional loss term associated with the estimation of the initial state, and it satisfies
\[ \frac{\partial \mathcal L}{\partial A} = 0, \ \frac{\partial \mathcal L}{\partial C} = 0. \]
This case can be considered as the deterministic counterpart of the framework used in variational inference of state space models \citep{jordan1999introduction, archer2015black}. This comparison is discussed further in Section \ref{sec:discussion}.
\end{enumerate}

In the following sections, we will demonstrate the analysis and state the theorems for problem (\ref{eqn:model-1}) in the second case. The statements are identically valid for the other two cases, as explained in Appendix \ref{app:initial}.

% \vspace{1in}
% Given some fixed parameter $\theta$, whether the following problem
% \begin{subequations}
% \label{eqn:basic-problem}
% \begin{align}
% \underset{\{z_k(0)\}_{k \in \mathcal K}}{\text{minimize}} & \quad \sum_{k \in \mathcal K} \sum_{t\in \mathcal T_k}  {\left(x_k(t) - c_\theta(z_k(t))\right)}^2 \\
% \text{subject to} & \quad \frac{dz_k}{dt} = f_{\theta}(z_k) \quad \forall k \in \mathcal K
% \end{align}
% \end{subequations}
% reaches zero or not shows how well the trajectories of samples could be represented with the common dynamics $f_{\theta}$.
% If the parameter $\theta$ is allowed to be changed within some class $\Theta$, we can optimize over $\theta$ to find the best model in $\{(f_\theta, c_\theta)\}_{\theta \in \Theta}$ for which there exists a set of initial states to represent all the samples $\{x_k(t)\}_{t \in \mathcal T_k, k \in \mathcal K}$ with a common dynamics, and this leads us to the problem (\ref{eqn:model-1}):

% \begin{subequations}
% \label{eqn:model-1}
% \begin{align}
% \underset{\theta \in \Theta,\, \{z_k(0)\}_{k \in \mathcal K}}{\text{minimize}}
%     % \min_{\theta \in \Theta} \min_{\{z_k(0) : k \in \mathcal K \}} 
%     & \quad  \sum_{k \in \mathcal K} \sum_{t\in \mathcal T_k} {\left(x_k(t) - c_\theta( z_k(t))\right)}^2 \\
% \text{subject to \ \ } & \quad \frac{dz_k}{dt} = f_{\theta}(z_k) \quad \forall k \in \mathcal K.
% \end{align}
% \end{subequations}

\section{Learning with Squared-Error Loss}
\label{sec:sec-lti}

In this section, we consider problem (\ref{eqn:model-1}) with the squared-error loss:
\begin{align*}
    \underset{A, C, \{s_k\}_{k \in \mathcal K}}{\text{minimize}} & \quad \sum_{k \in \mathcal K} \sum_{t \in \mathcal T_k} {\| x_k(t) - C e^{At}s_k \|}_2^2 
    % \\
    % \text{subject to} & \quad \frac{dz_k}{dt} = A z_k \qquad \forall k \in \mathcal K, \ \forall t \ge 0, \\
    % & \quad z_k(0) = s_k \qquad \  \forall k \in \mathcal K.
\end{align*}
% \[ \ell( x_k(t), C z_k(t) ) = {\| x_k(t) - C z_k(t) \|}_2^2 .\]
% \[ \frac{dz}{dt} = f_\theta(z) = Az \quad \forall z \in \mathbb R^n\]
% with the observation function
% \[ c(z) = z \quad \forall z \in \mathbb R^n,\]
% where $A \in \mathbb R^{n \times n}$ is the matrix describing the state transition of the system. Given the initial states $\{s_k\}_{k \in \mathcal K}$, the trajectories for this system can be written as \citep{callier1991}
% \[ z_k(t) = e^{At}s_k \quad \forall k \in \mathcal K. \]
If we use the gradient descent algorithm to solve problem (\ref{eqn:model-1}), the learning rate of the algorithm needs to be sufficiently small for the algorithm to converge \citep{bertsekas-nonlinear}. 
The next theorem gives an upper bound on the learning rate as a necessary condition for the convergence of the algorithm.

\begin{theorem}
\label{theorem:step-size}
Let $\{z_k\}_{k \in \mathcal K}$ be a set of trajectories, and let $s_k$ denote the initial state for trajectory $z_k$ for each $k \in \mathcal K$. Define the set of sampling instants of $z_k$ as $\mathcal T_k$, and denote the samples taken from this trajectory by $\{x_k(t)\}_{t \in \mathcal T_k}$. Assume that  the gradient descent algorithm is used to solve the problem
\begin{align}
\label{problem:linear}
    \min_{A, C,  \{s_k\}_{k \in \mathcal K}} & \quad \sum_{k \in \mathcal K} \sum_{t \in \mathcal T_k} \left\| x_k(t) - Ce^{At}s_k \right\|_2^2.
\end{align}
Then for almost every initialization, the learning rate of the gradient descent algorithm, $\delta$, must satisfy
\[ \delta \le \frac{2}{ \lambda_{\min} \left( \rho^2 \sum_{k \in \mathcal K} \sum_{t \in \mathcal T_k} t^2 e^{2 \text{\emph{Re}}(\Lambda) t} \hat s_k \hat s_k^\top \right)} \]
so that the algorithm can converge to the solution $\big(\hat A, \hat C,  \{\hat s_k\}_{k\in \mathcal K} \big)$ achieving zero training error,
where $\lambda_\text{\emph{min}}(\cdot)$ denotes the minimum eigenvalue of its argument, $\Lambda$ is the eigenvalue of $\hat A$ with the largest real part, $\rho^2 = \max_{u \in \mathcal U} \|\hat Cu\|_2^2$, and $\mathcal U$ is the set of eigenvectors of $\hat A$ corresponding to $\Lambda$.
\end{theorem}
\begin{proof}
See Appendix \ref{appendix:proof-theorem-step-size}.
\end{proof}

Note that the eigenvalues of a linear dynamical system have a particular meaning in control theory: they describe the stability of the system \citep{callier1991}. If any eigenvalue of $\hat A$ has a real part that is strictly positive, then the state of the system will grow unboundedly large from almost all initial points; and the system is called unstable in this case. If, on the other hand, all eigenvalues of $\hat A$ has a negative real part, then the state of the system will converge to a fixed point from all initial points, and the system will be stable.

% \noindent
% \textbf{Remark 1:}
The condition about reaching zero training error might be somewhat restrictive, but the main purpose of Theorem \ref{theorem:step-size} is not to prescribe a learning rate for all possible cases; it is to reveal that the samples taken at different times affect the convergence of the algorithm very differently. Indeed, Theorem \ref{theorem:step-size} shows that if the gradient descent algorithm is used to learn an unstable system, samples taken at later times impose a bound on the required learning rate exponentially more strict, which renders learning an unstable dynamical system infeasible.

Note that if the set of initial states $\{\hat s_k\}_{k \in \mathcal K}$ does not span the whole state space, then the bound given in Theorem \ref{theorem:step-size} will be void. This suggests that it will be easier to train a dynamical model if the initial states of the trajectories given in the training data do not have a large variance. However, this does not mean the learned model will be accurate. Since there is no information available about how the system evolves for the initial states in the nullspace of $\sum_{k \in \mathcal K} \hat s_k \hat s_k^\top$, the model learned will fail to predict the behavior of the system for the initial states with a nonzero component in this unlearned subspace as well.

The appearance of $\rho$ in Theorem \ref{theorem:step-size} reflects the notion of observability \citep{callier1991}. Based on the relationship between the matrices $\hat A$ and $\hat C$, it may not be possible to observe certain eigenvalues, or modes, of the learned system in its output; these modes are called unobservable modes. As these modes do not appear in the output of the learned system, they cannot affect the gradient descent algorithm. 
% Note that as long as the state-space of the model is not larger than what is needed, all modes of the learned system will be observable, and $\rho$ will be nonzero.

\begin{remark} The analysis for Theorem \ref{theorem:step-size} shows that, for the Hessian $H$ of the loss function (\ref{problem:linear}) at $(\hat A, \hat C)$, the ratio of the largest eigenvalue to the smallest eigenvalue of $H$ satisfies
\[ \frac{\lambda_{\max}(H)}{\lambda_{\min}(H)} \ge \frac{
\lambda_{\min} \left( \rho_1^2 \sum_{k \in \mathcal K} \sum_{t \in \mathcal T_k} t^2 e^{2 \text{\emph{Re}}(\lambda_1) t} \hat s_k \hat s_k^\top \right)
}{
\lambda_{\max} \left( \rho_2^2 \sum_{k \in \mathcal K} \sum_{t \in \mathcal T_k} t^2 e^{2 \text{\emph{Re}}(\lambda_2) t} \hat s_k \hat s_k^\top \right)
}
\]
for any pair of eigenvalues $(\lambda_1, \lambda_2)$ of $\hat A$, where $\rho_1 = \|\hat C u_1\|_2$, $\rho_2 = \|\hat C u_2\|_2$, and $u_1$, $u_2$ are the right eigenvectors of $\hat A$ corresponding to $\lambda_1$, $\lambda_2$, respectively. This implies that, if the loss function can be represented well by its second order approximation around $(\hat A, \hat C)$, local convergence rate for estimating the eigenvalue $\lambda_2$ will require 
\[O\left( \left[{\log\left( \left( 1 - \beta
\frac{\sum_{k\in \mathcal K} \sum_{t \in \mathcal T_k} t^2 e^{2 Re(\lambda_2)t}}
{\sum_{k\in \mathcal K} \sum_{t \in \mathcal T_k} t^2 e^{2 Re(\lambda_1)t}}
\right)^{-1} \right)
} \right]^{-1}\right)\]
iterations of the gradient descent algorithm, where $\beta$ is some constant depending on $\rho_1, \rho_2$ and $\sum_{k\in \mathcal K} \hat s_k \hat s_k^\top$. This shows that  learning the stable modes of a system can become infeasible when the system is unstable. See Appendix \ref{app:conv-rate} for more details.
\end{remark}

The necessary condition given in Theorem \ref{theorem:step-size} implies that the convergence of the algorithm gives us information about the rightmost eigenvalue of the dynamical system that is being estimated. This is stated in Corollary \ref{first-corol}.

\begin{corollary}
\label{first-corol}
Assume that the observation matrix $C = I$, the gradient descent algorithm is used to solve the problem (\ref{problem:linear}) and the algorithm has converged from a random\footnote{The random distribution is assumed to assign zero probability to every set with Lebesgue measure zero.} initialization to the solution $\big(\hat A, \{\hat s_k\}_{k\in \mathcal K} \big)$ achieving zero training error. Then the eigenvalue of $\hat A$ with the largest real part, $\Lambda$, almost surely satisfies
% \[
%  e^{2 \text{\emph{Re}}(\Lambda) \tau}
%  \le \frac{1}{\delta \tau^2}
% \frac{2}{\lambda_{\min}\left( \sum_{k \in \mathcal K} \sum_{t \in \mathcal T_k} \hat s_k \hat s_k^\top \mathbf{1}_{\{t \ge \tau\}} \right)} \quad \forall \tau>0
% \]
\[
 \text{\emph{Re}}(\Lambda)
 \le \inf_{\tau > 0} \frac{1}{2 \tau} \log \left[ \frac{1}{\delta \tau^2}
\frac{2}{\lambda_{\min}\left( \sum_{k \in \mathcal K} \sum_{t \in \mathcal T_k} \hat s_k \hat s_k^\top \mathbf{1}_{\{t \ge \tau\}} \right)} \right] \hspace{0.65in}
\text{if} \ \  \text{\emph{Re}}(\Lambda) >0, \]
% \[
%  e^{2 \text{\emph{Re}}(\Lambda) \tau_{\max}} \le \frac{1}{\delta \tau_{\min}^2}
% \frac{2}{\lambda_{\min}\left( \sum_{k \in \mathcal K} \sum_{t \in \mathcal T_k} \hat s_k \hat s_k^\top 
% \mathbf{1}_{\{\tau_{\min} \le t \le \tau_{\max}\}} \right)} \quad \forall \tau_{\max} > \tau_{\min}>0
% \]
% \[ \text{\emph{Re}}(\Lambda) \le
% \inf_{\tau_{\max} > \tau_{\min} > 0} \frac{1}{2 \tau_{\max}} \log\left[
% \frac{1}{\delta \tau_{\min}^2}
% \frac{2}{\lambda_{\min}\left( \sum_{k \in \mathcal K} \sum_{t \in \mathcal T_k} \hat s_k \hat s_k^\top 
% \mathbf{1}_{\{\tau_{\min} \le t \le \tau_{\max}\}} \right)} \right]
% \]
\[ \text{\emph{Re}}(\Lambda) \le
\inf_{\tau_{2} > \tau_{1} > 0} \frac{1}{2 \tau_{2}} \log\left[
\frac{1}{\delta \tau_{1}^2}
\frac{2}{\lambda_{\min}\left( \sum_{k \in \mathcal K} \sum_{t \in \mathcal T_k} \hat s_k \hat s_k^\top 
\mathbf{1}_{\{\tau_{1} \le t \le \tau_{2}\}} \right)} \right]
\quad \text{if} \ \  \text{\emph{Re}}(\Lambda) < 0.\]

\end{corollary}

% Theorem \ref{theorem:step-size} and Corollary \ref{first-corol} show that learning a dynamical system is much easier if the system is stable, that is, if the real parts of its eigenvalues are negative. 
% This result is closely related to the conclusion of \citep{nar2018step}, which shows that when training a deep linear network with the gradient descent algorithm, it is easier for the algorithm to learn a linear mapping with a small Lipschitz constant. When we learn a linear dynamical system from the whole trajectory in an offline fashion, the linear mapping to be learned is repeatedly applied to the initial state. Consequently, the largest eigenvalue of the  system becomes the critical value that determines the Lipschitz constant of the mapping that takes the initial state and produces the observations. It seems that what is advantageous for training feedforward networks and what helps their generalization through implicit regularization becomes an obstacle in learning unstable dynamical systems.

\section{Learning with Time-Weighted Logarithmic Loss}
\label{sec:section4}

Theorem \ref{theorem:step-size} shows that when the gradient descent algorithm is used to learn the parameters of an unstable dynamical system, the effect of the samples taken at later times are exponentially more weighted around a global minimum. It is important to note that this is the case for the choice of squared-error loss as the training loss function. In this section, we introduce a new loss function in order to balance the effects of all samples on the dynamics of the algorithm. This new loss function greatly relaxes the necessary condition given in Theorem 1, and it enables training even unstable linear systems with the gradient descent algorithm.

For any $\epsilon > 0$, define $F_{\epsilon}: \mathbb R \to \mathbb R$ as
\begin{equation} F_{\epsilon}(\xi) = \left\{ \begin{array}{r l}
\log(\epsilon + \xi) - \log(\epsilon) & \quad \xi \ge 0,\\
-\log(\epsilon - \xi) + \log(\epsilon) & \quad \xi < 0.
\end{array}
\right.
\label{define:F}
\end{equation}
Given two trajectories $\{x(t)\}_{t \in \mathcal T}$ and $\{y(t)\}_{t \in \mathcal T}$ in $\mathbb R^n$, consider the loss function defined as
\[ \ell(x, y) = \sum_{t \in \mathcal T} \sum_{j=1}^n \frac{1}{t^2} \left( F_\epsilon(e_j^\top x(t)) - F_\epsilon(e_j^\top y(t))
\right)^2,
\]
where $e_j$ denotes the $j$-th standard basis vector with a 1 in its $j$-th coordinate and 0 in all other coordinates.
Note that $\ell(x, y)$ is zero if and only if $x(t)=y(t)$ for all $t \in \mathcal T$; and it is strictly positive otherwise. Similar to Section \ref{sec:sec-lti}, we will analyze this loss functions for learning linear dynamical systems.

%yx: please rearrange the eq \ref{pre-loss}so that the learning rate is on the left side of the inequality as in Theorem 1. 
\begin{theorem}
\label{second-theorem}
Let $\{z_k\}_{k \in \mathcal K}$ be a set of trajectories, and let $s_k$ denote the initial state for trajectory $z_k$ for each $k \in \mathcal K$. Define the set of sampling instants of $z_k$ as $\mathcal T_k$, and denote the samples taken from this trajectory by $\{x_k(t)\}_{t \in \mathcal T_k}$. Assume that  the gradient descent algorithm is used to solve 
\begin{equation}
\label{pre-loss}
    \min_{A, C, \{s_k\}_{k \in \mathcal K}} \ \sum_{k \in \mathcal K} \sum_{t \in \mathcal T_k} \sum_{j=1}^n \frac{1}{t^2} \left( F_\epsilon(e_j^\top x_k(t)) - F_\epsilon(e_j^\top C e^{At}s_k) \right)^2,
\end{equation}
where $F_\epsilon$ is as defined in (\ref{define:F}). Then for almost every initialization, the learning rate $\delta$ of the gradient descent algorithm must satisfy
\[ \delta \le \frac{2}{\lambda_\text{\emph{min}} \left( \sum_{k \in \mathcal K} \sum_{t \in \mathcal T_k} \frac{ \rho^2 e^{2\text{Re}(\Lambda)t}}{  {({\|\hat C e^{\hat At} \hat s_k\|}_\infty + \epsilon)}^2} \hat s_k \hat s_k^\top \right)}\]
so that the algorithm can converge to the solution $(\hat A, \hat C,  \{\hat s_k\}_{k \in \mathcal K})$ achieving zero training error, where $\Lambda$ is the eigenvalue of $\hat A$ with the largest real part, $\rho^2 = \max_{u \in \mathcal U} \|\hat C u\|_2^2$, and $\mathcal U$ is the set of right-eigenvectors of $\hat A$ corresponding to its eigenvalue $\Lambda$.
\end{theorem}

\begin{proof} See Appendix \ref{appendix:proof-adapted-loss}. \end{proof}

The necessary conditions on the step size given in Theorem \ref{second-theorem}
%yx:double check should this be Theorem 1 or 2? %\ref{second-theorem}
and in Theorem \ref{theorem:step-size} are obtained by following the identical analysis procedure. Theorem \ref{second-theorem} shows that the loss function (\ref{pre-loss}) substantially relaxes the necessary condition given in Theorem \ref{theorem:step-size}, and it balances the weights of all the sampling instants on the dynamics of the gradient descent algorithm. In other words, it makes it easier for the gradient descent algorithm to converge to the global minima. This is demonstrated in the next section.

% Furthermore, it normalizes the norms of the initial states, thereby giving equal importance to each trajectory during training without requiring any knowledge about these initial states.

\section{Experiments}

To check if the time-weighted logarithmic loss function introduced in Theorem \ref{second-theorem} allows learning linear dynamical systems with the gradient descent algorithm, we generated a set of output trajectories from randomly generated linear systems and trained a linear model with this data set by using the logarithmic loss function. We also trained the model with the same data set by using the mean-squared-error loss to compare the two estimates.

For the experiments, we considered the discretized version of the dynamical systems. In other words, we used
\[ z_k(t) = A^t z_k(0) \quad \forall t \in \mathbb N, \ \forall k \in \mathcal K.\]
Note that with this discrete-time representation, the stability of the system is described based on the position of the eigenvalues relative to the unit circle. The system is stable if all of its eigenvalues are inside the unit circle.

We randomly generated $A \in \mathbb R^{n\times n}$ and $C \in \mathbb R^n$ to produce a set of observation sequences. In particular, we generated $A$ as
$A = I + \Delta A$,
where $\Delta A$ is a matrix whose entries are independent and uniformly distributed between $[-0.5, 0.5]$. The elements of $C$ were drawn from independent standard normal distributions. We obtained 50 trajectories from the generated system by providing different initial states, and each trajectory consisted of 50 observations. 
% The implementation is available in Supplementary Material.

For training a linear model on this data set, we used the stochastic gradient method with momentum. Both for the mean-squared-error loss and for the time-weighted logarithmic loss, the gradients were normalized to unit norm if their $\ell_2$ norm exceeded 1. Figure \ref{fig:training-error} shows a typical plot for the training error of an unstable system for each of these loss functions. We observe that the gradient descent algorithm is not able to decrease the mean-squared error loss, whereas the time-weighted logarithmic loss function is diminished easily.

\begin{figure}[ht]
    \centering
    \includegraphics[scale=0.45]{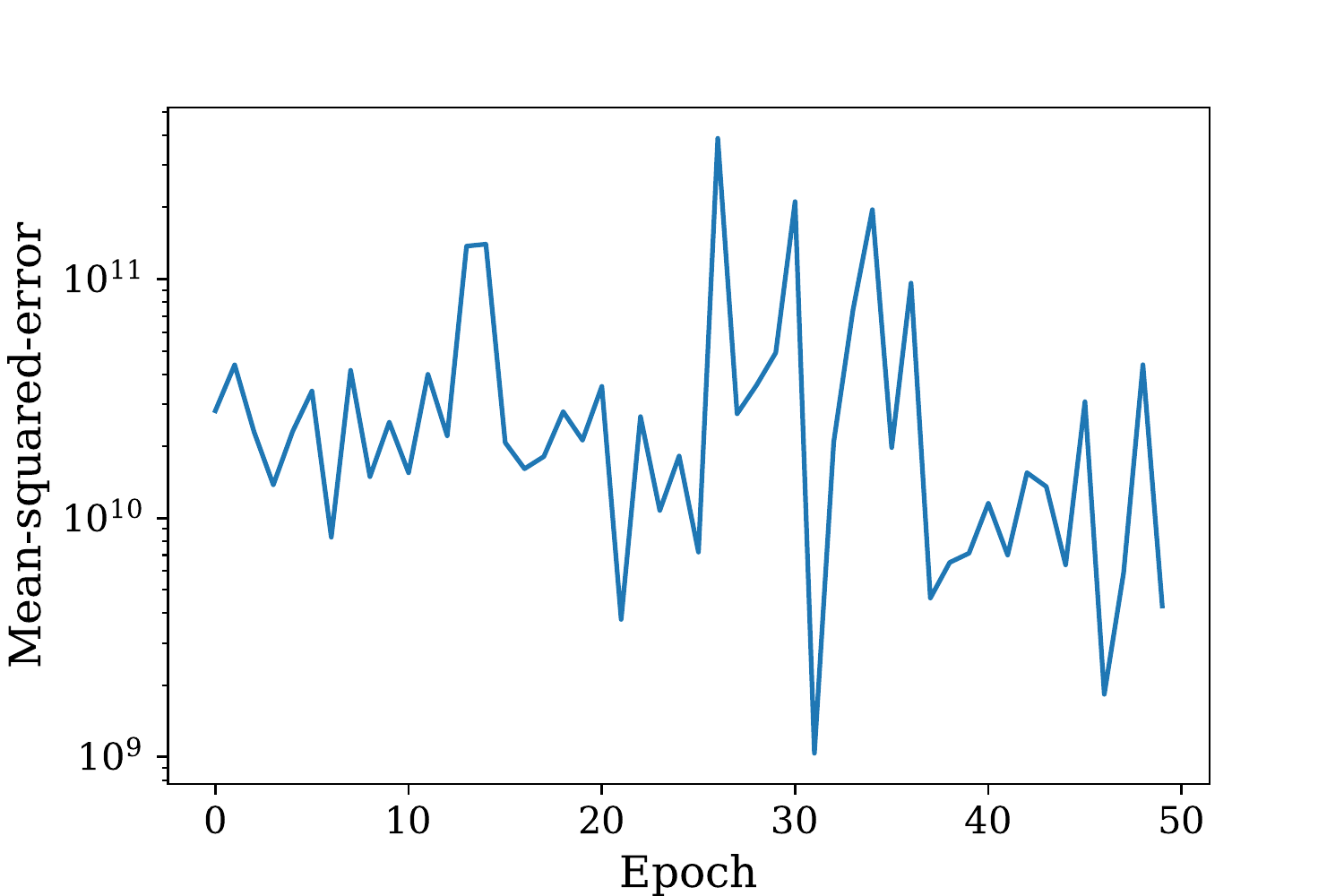}
    \includegraphics[scale=0.45]{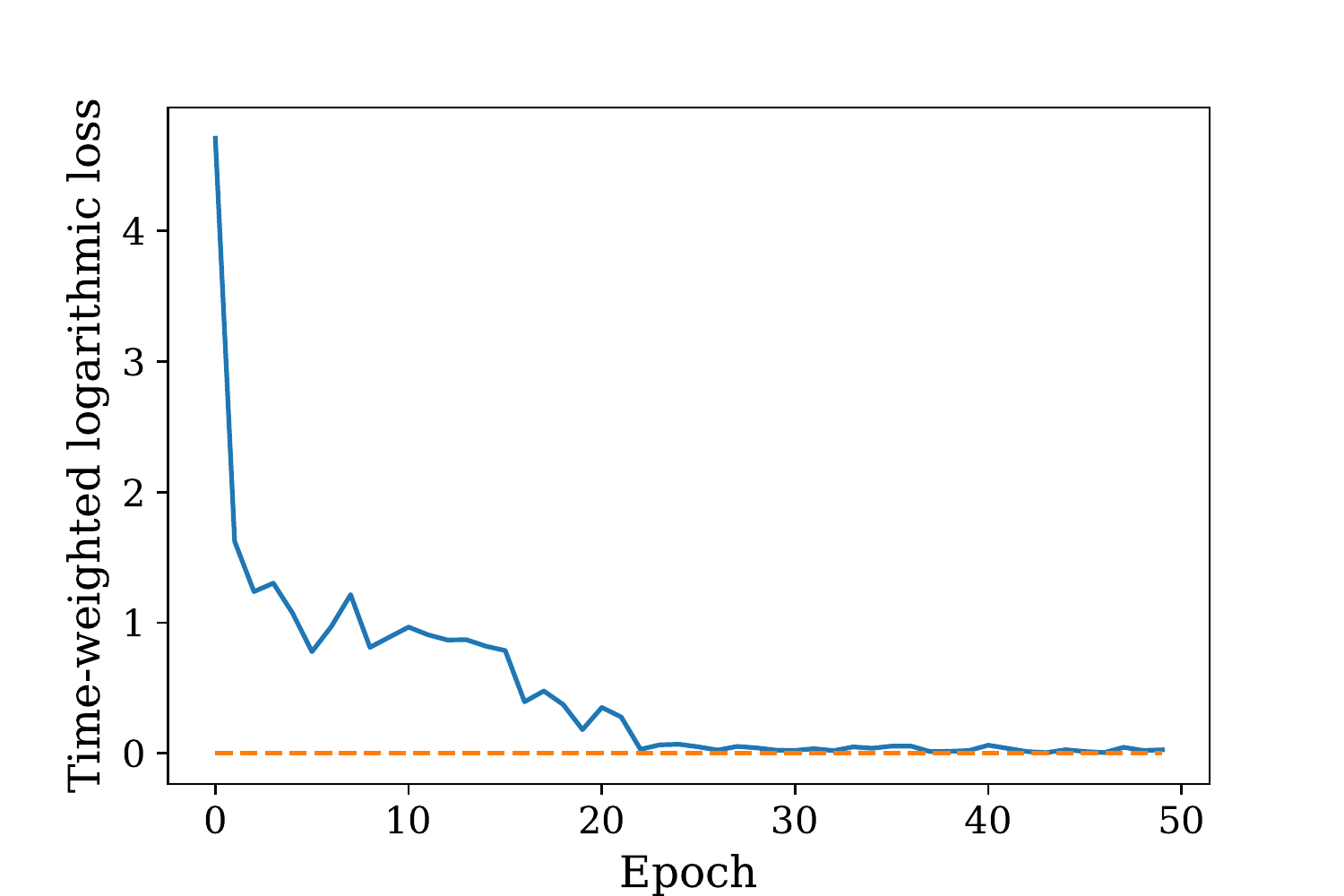}
    \caption{Typical plots of training error when mean-squared-error is used [left] and when time-weighted logarithmic loss function is used [right].}
    \label{fig:training-error}
\end{figure}

To check if this decrease in the loss function corresponds to an effective learning of the actual model, we computed the eigenvalues of the estimated system throughout training and compared them with the eigenvalues of the actual system.
% The reason why we considered only the eigenvalues is that given a sequence of observations from a linear system, the model parameters $A$ and $C$ are not unique. On the other hand, the eigenvalues of $A$ corresponding to its observable modes can be recovered given .
Figure \ref{fig:seed5}
% and Figure \ref{fig:seed22} 
demonstrates an example of how the estimates for the eigenvalues evolve during training. The state space of the system in Figure \ref{fig:seed5} is three dimensional, and the system is unstable as one of its eigenvalues is outside of the unit circle. When the mean-squared-error loss is used, only the unstable mode of the system is estimated correctly. In contrast, the time-weighted logarithmic loss function is able to discover all three modes of the system. Additional plots are provided in Appendix \ref{sec:add-experiment}.
% Similarly, the system in Figure \ref{fig:seed22} has a four-dimensional state space, and the mean-squared-error loss fails at learning the eigenvalues of this system. On the other hand, the time-weighted logarithmic loss function finds all of the eigenvalues accurately.

\begin{figure}[ht]
\centering
\subfigure[Eigenvalues with mean-squared-error]{
\includegraphics[trim={2cm 0 2cm 0}, clip, scale=0.6]{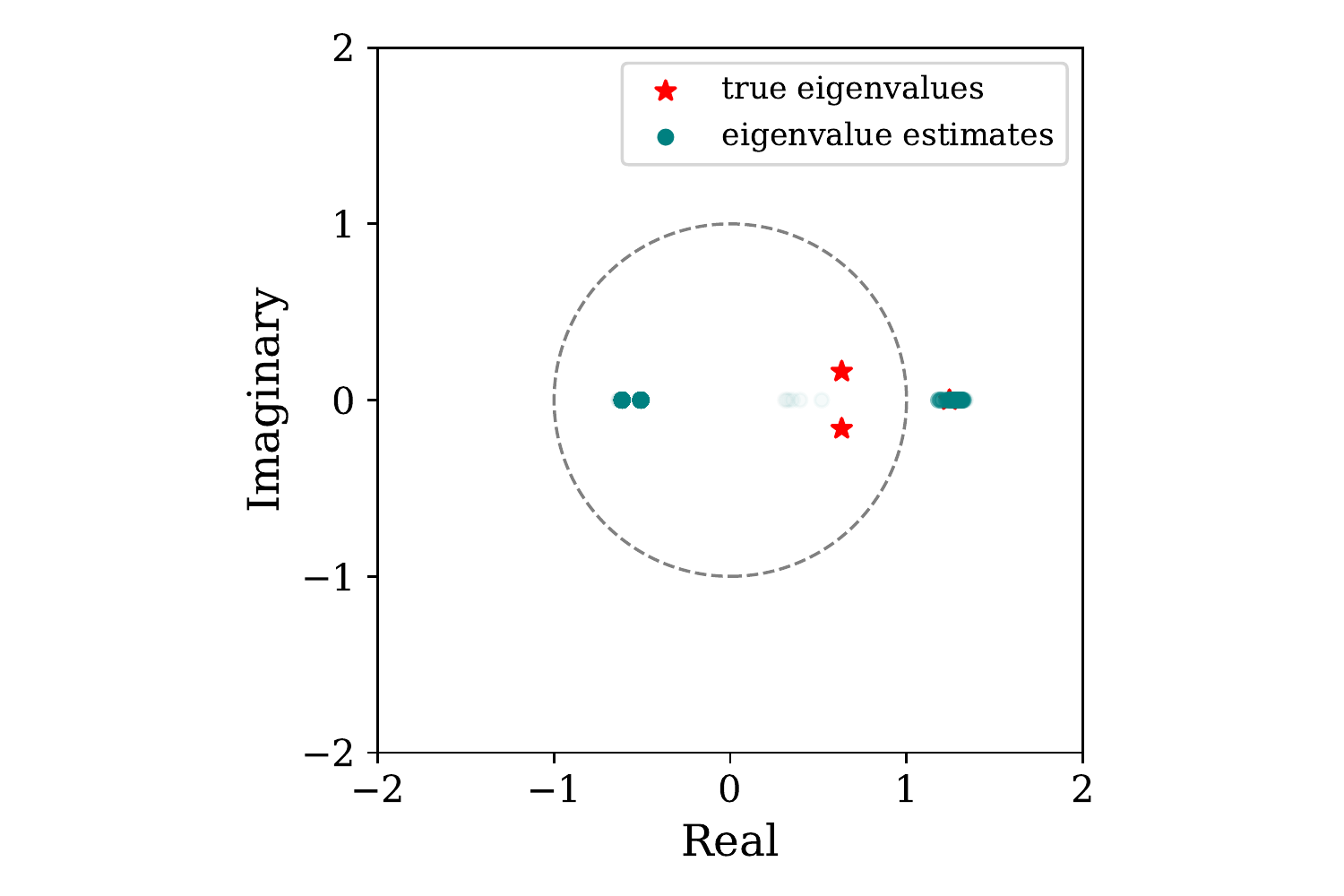}
}
\subfigure[Eigenvalues with logarithmic loss]{
\includegraphics[trim={2cm 0 2cm 0}, clip, scale=0.6]{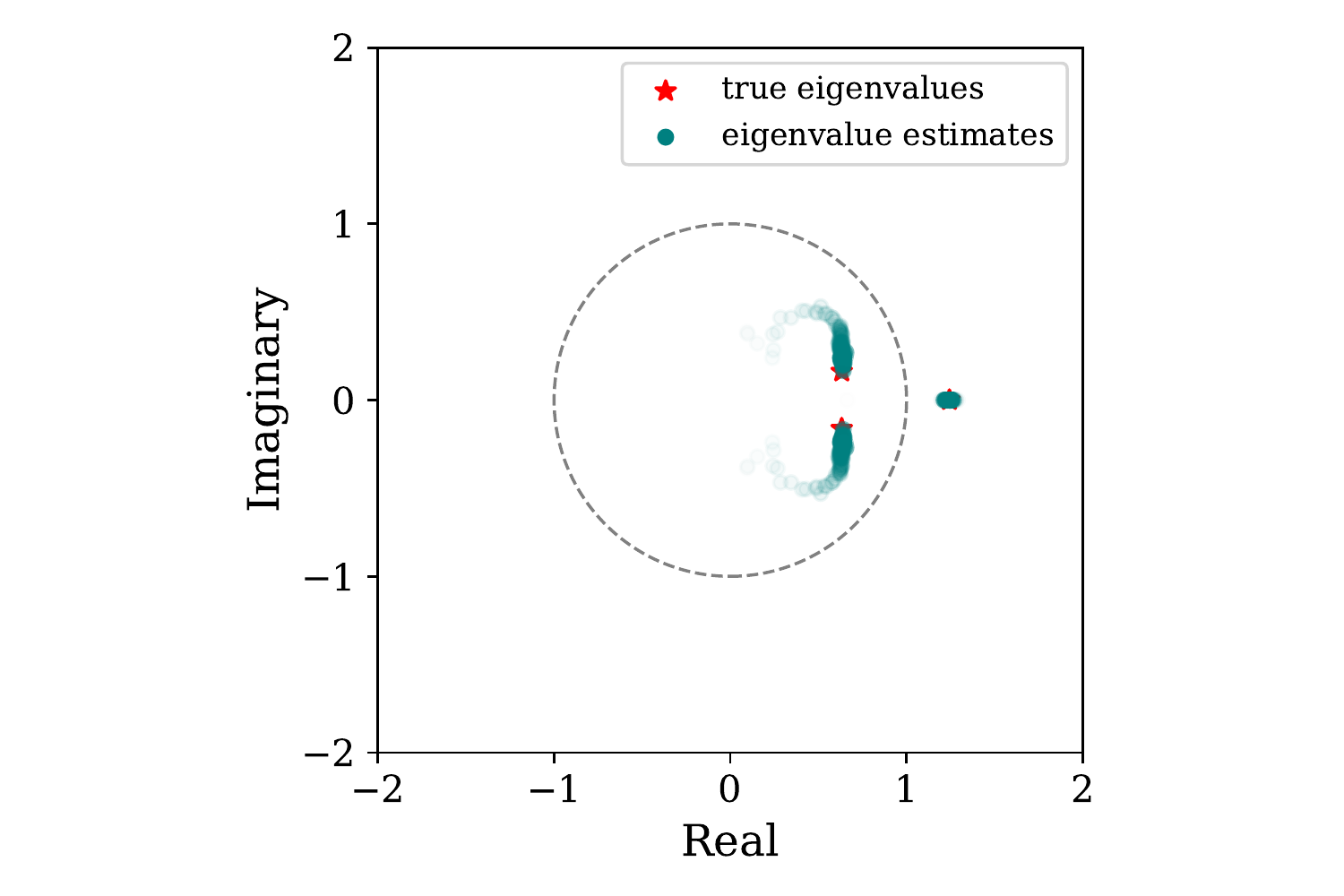}
}
\caption{A linear system with three-dimensional state space is trained with mean-squared-error loss [left] and time-weighted logarithmic loss [right]. The red stars show the eigenvalues of the real system, whereas the green dots show the eigenvalues of the estimated system. Earlier estimates of the eigenvalues are depicted with faded colors. Mean-squared-error loss is able to find only the unstable mode, whereas the logarithmic loss function discovers all three modes correctly.}
\label{fig:seed5}
\end{figure}

% \begin{figure}[ht]
% \centering
% \subfigure[Eigenvalues with mean-squared-error]{
% \includegraphics[trim={2cm 0 2cm 0}, clip, scale=0.6]{figs/pres-mse-evals-22.pdf}}
% \subfigure[Eigenvalues with logarithmic loss]{
% \includegraphics[trim={2cm 0 2cm 0}, clip, scale=0.6]{figs/pres-logloss-evals-22.pdf}
% }
% \caption{A linear system with four-dimensional state space is trained with mean-squared-error loss [left] and time-weighted logarithmic loss [right]. The red stars show the eigenvalues of the real system, whereas the green dots show the eigenvalues of the estimated system. Earlier estimates of the eigenvalues are depicted with faded colors. Mean-squared-error loss conspicuously fails to learn the system model, whereas the logarithmic loss function estimates all of the eigenvalues accurately.}
% \label{fig:seed22}
% \end{figure}

\section{Discussion}
\label{sec:discussion}

\noindent \textbf{Variational inference.} Variational inference is a Bayesian approach to handle the unknown parameters and the unobserved states of a dynamical system simultaneously \citep{jordan1999introduction, archer2015black}. For variational inference, the system is described by a generative model: $p_\theta(x, z)$, where $x = \{x(t)\}_{t \in \mathcal T}$ and $z = \{z(t)\}_{t \in \mathcal T}$ are the sequence of observations and hidden states of the system, and $\theta$ is the parameters of the model. Given the observations, the posterior is approximated by another model: $g_\phi(\cdot | x)$. Then, the objective function to be minimized is described as \citep{archer2015black}
\begin{equation}
\label{loss:variational}
- \mathcal H(g_\phi(\textbf{z} | x)) - {\mathbb E}_{g_\phi(\textbf{z}|x)}[\log(p_\theta(x,\textbf{z})], \end{equation}
where $\mathcal H$ is the entropy of its argument. Assume the stochasticity of the initial state and the state transitions is removed, and each observation $x(t)$ is obtained through an observation mapping with an additive Gaussian noise:
\[ x(t) =  c(z(t)) + \xi_t,\]
where $\{\xi_t\}_{t \in \mathcal T}$ is an independent and identically distributed sequence. Then the minimization of the loss function (\ref{loss:variational}) reduces to the problem
\begin{align*}
\underset{\theta, \phi}{\text{minimize}} & \quad \sum\nolimits_{t \in \mathcal T} {\left\| x(t) - c(z(t)) \right\|}_2^2 \\
\text{subject to} & \quad z(0) = \argmax_{\tilde z} g_\phi(\tilde z | x),
\end{align*}
and the system identification problem becomes equivalent to problem (\ref{model:third}). This is the reason why we referred to (\ref{model:third}) as the deterministic counterpart of the variational inference formulation.

\noindent \textbf{Random initial states.} In our analyses, we treated the initial states as unknown but deterministic values that could be learned during training. With this deterministic viewpoint, Theorem~\ref{theorem:step-size} and Theorem~\ref{second-theorem} assumed that the observations could be matched to the latent state of the system perfectly and the training loss function could be made identically zero. It is not possible in general to satisfy this requirement with an expected loss over a set of random initial states. Therefore, Theorem~\ref{theorem:step-size} and Theorem~\ref{second-theorem} do not apply to the formulations with random initial states verbatim.

\noindent \textbf{Convergence of policy gradient.} Even though the focus of this work has been on system identification, the gradient descent algorithm will exhibit similar convergence problems when maximizing an objective over a time horizon while altering the dynamics of a dynamical system. Note that policy gradient methods in reinforcement learning \citep{sutton2018reinforcement} fall into this category. This is why our analysis in this work can potentially be used for studying and improving the stability of policy gradient methods.

\section{Conclusion}

To understand the hardness of learning dynamical systems from observed trajectories, we analyzed the dynamics of the gradient descent algorithm while training the parameters of a dynamical model, and we observed that samples taken at different times affect the dynamics of the algorithm in substantially different degrees. To balance the effect of samples taken at different times, we introduced the time-weighted logarithmic loss function and demonstrated its effectiveness.

In this work, we focused on learning linear dynamical systems. Whether a similar loss function improves training of nonlinear models is an important direction for future research. In addition, we considered a deterministic framework for our problem formulation with a dynamical system. An interesting question is whether allowing randomness in the state of the system or the state transitions could trade off the accuracy of the estimated model for the efficiency of the training procedure. 

% \section{Broader Impact}
% This work does not present any foreseeable societal consequence.

% For applications that require safety or try to minimize potential risks, it is essential that the system model be estimated accurately. Failure to learn the stable modes of a linear dynamical system, for examples, leads to wrong predictions in the short term. If these predictions are used for decision making, the decisions that are not optimal for the actual system can cause undesired outcomes in the long term as well. We showed that the time-weighted logarithmic loss function enables learning both the stable and the unstable modes of a linear system accurately.

\bibliography{references}

\appendix

\section{Proof of Theorem \ref{theorem:step-size}}
\label{appendix:proof-theorem-step-size}
To begin with, assume that $C$ is a fixed matrix, and consider only one trajectory $z$ with only one sample taken at time $t$. 
Then the loss function to be minimized is
\[ \ell(A,s) = \frac{1}{2} \|x - C e^{At}s \|_2^2, \]
where $s$ denotes the initial state of the trajectory.
The update rule for the gradient descent algorithm gives
\begin{subequations}
\label{eqn:grad-desc}
\begin{align}
A & \leftarrow A - \frac{\delta}{2}\frac{\partial }{\partial A} \langle C e^{At}s -x , C e^{At}s - x \rangle \\
s & \leftarrow s - \frac{\delta}{2}\frac{\partial }{\partial s} \langle C  e^{At}s -x , C e^{At}s - x \rangle
\end{align}
\end{subequations}
This update rule creates a nonlinear dynamical system where the state of the system is the parameters $(A, s)$.

A dynamical system can converge to its equilibrium only if that equilibrium is stable in the sense of Lyapunov.
A standard tool to analyze the stability for nonlinear systems is given by Lyapunov's direct method: an equilibrium of a nonlinear system can be stable only if the linearization of the system around that equilibrium has no unstable mode \citep{khalil}. If, on the other hand, the linearized model has an eigenvalue larger than 1 in magnitude, then the nonlinear system is definitely unstable --- which rules out the possibility of convergence to this equilibrium from its neighbors, except for a set on a low-dimensional manifold, which has Lebesgue measure zero.
This shows that the system (\ref{eqn:grad-desc}) can converge to an equilibrium only if all eigenvalues of the linearized model around that equilibrium are less than 1 in magnitude.

We can write the linearization of (\ref{eqn:grad-desc}) around an equilibrium $(\hat A, \hat s)$ as
\begin{align*}
    \tilde A & \gets \tilde A - \delta f_1(\tilde A) - \delta f_2(\tilde s), \\
    \tilde s & \gets \tilde s - \delta f_3(\tilde A) - \delta f_4(\tilde s),
\end{align*}
where
\begin{itemize}
    \item $f_1$ is the Jacobian with respect to A of the gradient with respect to A of the loss function $\ell$ at $(\hat A, \hat s)$,
    \item $f_2$ is the Jacobian with respect to s of the gradient with respect to A of the loss function $\ell$ at $(\hat A, \hat s)$,
\end{itemize}
and $f_3$ and $f_4$ are defined similarly. Note that $f_2$ and $f_3$ are the Jacobians of the gradients of the same function with respect to the same parameters in different orders; therefore, they are hermitian of each other:
\[ \langle \tilde A, f_2(\tilde s) \rangle = \langle f_3(\tilde A) , \tilde s\rangle \quad \forall \tilde A, \forall \tilde s. \]
This shows that the linearized model can be associated with a symmetric matrix; and consequently, all of its eigenvalues are real-valued, and its eigenvalues can be less than 1 only if all of its diagonal blocks have eigenvalues less than 1. In other words, a necessary condition for the solution $(\hat A, \hat s)$ to be stable is that the mappings
\begin{align}
     \tilde A & \gets \tilde A - \delta f_1(\tilde A) \label{update:A} \\
   \tilde  s &  \gets \tilde s - \delta f_4(\tilde s) \label{update:s}
\end{align}
have eigenvalues less than 1 in magnitude, or equivalently, the functions $f_1$ and $f_4$ have eigenvalues less than $2/\delta$. Note that this conclusion would be identical if $C$ was also updated via the gradient descent algorithm. In particular, we would need the eigenvalue of the mapping $f_1$ to be less than 1 in magnitude around the equilibrium $(\hat A, \hat C, \{\hat s_k\}_{k \in \mathcal K})$.

% Since the loss function $\ell(A,s)$ is quadratic in the parameter $s$, the mapping $f_4$ can be written as:
% \[ f_4(\tilde s) =  e^{\hat A^\top t} e^{\hat A t} \tilde s. \]
% Then, its largest eigenvalue is given by
% \[ \max_{u : \|u\|_2 = 1} u^\top e^{\hat A^\top t}e^{\hat A t} u = e^{2 \Lambda t}, \]
% where $\Lambda$ is the largest eigenvalue of $\hat A$, and the equality is achieved by choosing $u$ to be the top eigenvector of $\hat A$. This implies that when we have multiple trajectories with multiple sampling points, we will have the constraint
% \[ \max_{k \in \mathcal K} \ \sum_{t \in \mathcal T_k} e^{2 \Lambda t} \le \frac{2}{\delta}, \]
% where $\mathcal T_k$ represents the set of sampling instants for trajectory $x_k$.

Finding a lower bound for the largest eigenvalue of the mapping $f_1$ will be easier with the following lemma.

\begin{lemma}
\label{first-lemma}
Let $f_i : \mathbb R^n \to \mathbb R$ be a twice-differentiable function for all $i \in \mathcal I$, and define
\[ F(x) = \frac{1}{2} \sum_{i \in \mathcal I} f_i^2(x). \]
If $F(x_0) = 0$, then the Hessian of $F$ at $x_0$ satisfies
\[ \nabla^2 F(x_0) = \sum_{i \in \mathcal I} \nabla f_i(x_0) \nabla f_i(x_0)^\top.
\]
\end{lemma}

\begin{proof} We can write the gradient and the Hessian of $F$, respectively, as
\begin{align*} \nabla F(x_0) & = \sum_{i \in \mathcal I} (\nabla f_i(x_0)) f_i(x_0),\\
\nabla^2 F(x_0) & = \sum_{i \in \mathcal I} \nabla f_i(x_0) \nabla f_i(x_0)^\top + f_i(x_0) \cdot \nabla^2 f_i(x_0).\end{align*}
Note that $F(x_0)=0$ implies that $f_i(x_0) = 0$ for all $i \in \mathcal I$. Then we have
\[ \nabla^2 F(x_0) = \sum_{i \in \mathcal I} \nabla f_i(x_0) \nabla f_i(x_0)^\top.\]
\end{proof}

Remember that $f_1(A)$ is the Jacobian with respect A of the gradient with respect to $A$ of the loss function
\[ \ell(A, C, s) = \frac{1}{2} \left\langle C e^{At}s -x,\, C e^{At}s - x \right\rangle. \]
Given $A \in \mathbb R^{n \times n}$, we can write
\[ \ell(A, C, s) = \frac{1}{2} \sum_{j =1}^n \left( e_j^\top C e^{At}s - e_j^\top x \right)^2, 
\]
where $e_j$ is the $j$-th standard basis vector with a 1 in its $j$-th coordinate and 0 in all other coordinates. Then, by using Lemma \ref{first-lemma}, the largest eigenvalue of the mapping $f_1$ can be lower bounded by
\begin{align}
\label{eqn:lower-bound-1}
\max_{Y : {\|Y\|}_F = 1} \ \sum_{j=1}^n \left|\left\langle Y, \nabla_A(e_j^\top C e^{At}s - e_j^\top x ) \right\rangle\right|^2.
\end{align}
To find the gradient, we can expand the matrix exponential:
\begin{align*}
    \nabla_A \left( e_j^\top C  \sum_{k=0}^\infty \frac{t^k}{k!}A^k s \right) = 
    \sum_{k=1}^{\infty} \sum_{r=0}^{k-1} \frac{t^k}{k!}{(A^\top)}^r C^\top e_js^\top {(A^\top)}^{k-1-r}.
\end{align*}
If we choose $\tilde Y = uv^\top$, where $u$ and $v$ are the unit-norm right and left eigenvectors of $A$ corresponding to its eigenvalue $\Lambda$ with the largest real part, we obtain
\begin{align*}
 \left\langle \tilde Y, \nabla_A \left( e_j^\top C \sum_{k=0}^\infty \frac{t^k}{k!}A^k s \right) \right\rangle & =
 \sum_{k=1}^\infty \sum_{r=0}^{k-1} \frac{t^k}{k!} \Lambda^{k-1} \langle u, C^\top e_j \rangle \langle v, s\rangle \\
 & = \sum_{k=1}^\infty \frac{t^k}{(k-1)!} \Lambda^{k-1} \langle u, C^\top e_j \rangle \langle v, s\rangle \\
 & = t e^{\Lambda t} \langle u, C^\top e_j \rangle \langle v, s\rangle \\
 & = t e^{\Lambda t} \langle Cu, e_j \rangle \langle v, s \rangle.
\end{align*}
Remember that (\ref{eqn:lower-bound-1}) is a lower bound for the largest eigenvalue of $f_1$, and so is
\begin{align*}
\sum_{j=1}^n {\left| \left\langle \tilde Y, \nabla_A(e_j^\top C e^{At}s - e_j^\top x ) \right\rangle \right|}^2 & =  \sum_{j=1}^n t^2 e^{2 \text{Re}(\Lambda) t} \left|\langle Cu,e_j \rangle \right|^2 \left|\langle v, s\rangle\right|^2 \\
& =  \rho^2 t^2 e^{2 \text{Re}(\Lambda) t} \left|\langle v, s\rangle\right|^2,
\end{align*}
where $\text{Re}(\Lambda)$ is the largest real part of the eigenvalues of $A$ and $\rho^2 = \|C u\|_2^2$.
If we have multiple trajectories, this lower bound will become
\[ \sum_{k \in \mathcal K} \sum_{t \in \mathcal T_k} \rho^2 t^2 e^{2 \text{Re}(\Lambda) t} \left|\langle v, s_k\rangle\right|^2, 
\]
where $\{s_k\}_{k \in \mathcal K}$ is the set of initial states of the trajectories.

As a result, for convergence of the gradient descent algorithm to a solution $(\hat A, \hat C, \hat s)$, it is necessary that
\[ \sum_{k \in \mathcal K} \sum_{t \in \mathcal T_k} \rho^2 t^2 e^{2\text{Re}(\Lambda) t} \left|\langle v, \hat s_k\rangle\right|^2 \le \frac{2}{\delta}.
\]
Without making any assumptions about the eigenvectors of $\hat A$, we can obtain the final necessary condition as
\[ \lambda_\text{min} \left( \sum_{k \in \mathcal K} \sum_{t \in \mathcal T_k} \rho^2 t^2 e^{2 \text{Re}(\Lambda) t} \hat s_k \hat s_k^\top \right) \le \frac{2}{\delta}, \]
or equivalently as
\[ \delta \le \frac{2}{\lambda_\text{min} \left( \rho^2  \sum_{k \in \mathcal K} \sum_{t \in \mathcal T_k} t^2 e^{2 \text{Re}(\Lambda) t} \hat s_k \hat s_k^\top \right)}. \]
This completes the proof. \hfill $\square$
    
\section{Proof of Theorem \ref{second-theorem}}
\label{appendix:proof-adapted-loss}
Similar to the proof of Theorem 1, we will use Lemma \ref{first-lemma} to find a lower bound for the largest eigenvalue of the linearized system around $(\hat A, \hat C,  \{\hat s_k\}_{k \in \mathcal K})$. Without loss of generality, assume $e_j C e^{At} s > 0$. Then,
\begin{align*} \nabla_A \log\left( e_j^\top C e^{At}s + \epsilon \right) &
= \nabla_A \log\left( e_j^\top C  \sum_{k=0}^\infty \frac{t^k}{k!}A^k s + \epsilon \right) \\
& = \frac{1}{ e_j^\top C e^{At}s + \epsilon} \sum_{k=1}^\infty \sum_{r=0}^{k-1} \frac{t^k}{k!} (A^\top)^r C^\top e_j s^\top (A^\top)^{k-1-r}.
\end{align*}
For the matrix $\tilde Y = uv^\top$, where $u$ and $v$ are the right and left eigenvectors of $A$ corresponding to its eigenvalue $\Lambda$ with the largest real part, we have
\begin{align*}
    \left\langle \tilde Y, \nabla_A \log\left( e_j^\top C e^{At}s + \epsilon \right) \right\rangle & = \frac{1}{ e_j^\top C e^{At}s + \epsilon} \sum_{k=1}^\infty \frac{t^k}{(k-1)!} \Lambda^{k-1} \langle u, C^\top e_j \rangle \langle v, s\rangle \\
    & = \frac{te^{\Lambda t}}{e_j^\top C e^{At} s + \epsilon}  \langle u, C^\top e_j \rangle \langle v, s\rangle. 
\end{align*}
By using Lemma \ref{first-lemma}, we obtain a lower bound for the largest eigenvalue of the linearization of the gradient descent algorithm around $(\hat A, \hat C,  \{\hat s_k\}_{k\in \mathcal K})$ as
\[
\sum_{k \in \mathcal K} \sum_{t \in \mathcal T_k} \sum_{j=1}^n \frac{1}{t^2} \left| \frac{te^{\Lambda t}}{e_j^\top C e^{At} s_k + \epsilon}  \langle C u, e_j \rangle \langle v, s_k\rangle \right|^2.
\]
We can write a further lower bound for this expression as
\begin{align*}
    \sum_{k \in \mathcal K} \sum_{t \in \mathcal T_k} \sum_{j =1}^n \frac{e^{2 \text{Re}(\Lambda) t}}{ \left(  {\| C e^{At}s_k\|}_\infty + \epsilon\right)^2} \left| \langle C u, e_j \rangle\right|^2 \left|\langle v, s_k\rangle\right|^2
    \\ = \sum_{k \in \mathcal K} \sum_{t \in \mathcal T_k} \frac{\rho^2 e^{2 \text{Re}(\Lambda) t}}{ \left( {\|C e^{At} s_k\|}_\infty + \epsilon\right)^2} \left|\langle v, s_k\rangle\right|^2,
\end{align*}
and finally,
\[ \lambda_\text{min} \left( \sum_{k \in \mathcal K} \sum_{t \in \mathcal T_k} \frac{\rho^2 e^{2\text{Re}(\Lambda) t}}{ \left(  {\|Ce^{At}s_k\|}_\infty + \epsilon\right)^2} s_k s_k^\top \right),\]
where $\rho^2 = \|\hat C u\|_2^2$ and $u$ is the right-eigenvector of $\hat A$ corresponding to its eigenvalue $\Lambda$.
For stability of the algorithm around the equilibrium point $(\hat A, \{\hat s_k\}_{k \in \mathcal K})$, we need
\[ \lambda_\text{min} \left( \sum_{k \in \mathcal K} \sum_{t \in \mathcal T_k} \frac{\rho^2 e^{2 \text{Re}(\Lambda) t}}{ \left(  {\|\hat C e^{\hat At} \hat s_k\|}_\infty + \epsilon\right)^2} \hat s_k \hat s_k^\top \right) \le \frac{2}{\delta},\]
where $\delta$ is the step size of the algorithm.

\section{Remarks on Convergence Rate}
\label{app:conv-rate}

In the proof of Theorem \ref{theorem:step-size}, we considered the mapping 
\[ \tilde A \gets \tilde A - \delta f_1(\tilde A), \]
where $f_1$ is the Jacobian of the gradient of the loss function
\[ \ell( A, s) = \frac{1}{2} \|x - C e^{At} s\|_2^2 \]
with respect to $A$ at the point $(\hat A, \hat C, \hat s)$. For Theorem \ref{theorem:step-size}, we computed the largest learning rate at which the algorithm can still converge to the specified equilibrium. Note that this was equivalent to computing a lower bound for the largest eigenvalue of the mapping $f_1$. Similar to the proof of Theorem \ref{theorem:step-size}, we can compute an upper bound for the smallest eigenvalue of $f_1$ around the solution $(\hat A, \hat C, \hat s)$.

% Given $A \in \mathbb R^{n \times n}$, we can write
% \[ \ell(A, C, s) = \frac{1}{2} \sum_{j =1}^n \left( e_j^\top C e^{At}s - e_j^\top x \right)^2, 
% \]
% where $e_j$ is the $j$-th standard basis vector with a 1 in its $j$-th coordinate and 0 in all other coordinates. Then, 

By using Lemma \ref{first-lemma}, the smallest eigenvalue of the mapping $f_1$ can be upper bounded by
\begin{align}
\label{eqn:upper-bound-1}
\min_{Y : {\|Y\|}_F = 1} \ \sum_{j=1}^n \left|\left\langle Y, \nabla_A(e_j^\top C e^{At}s - e_j^\top x ) \right\rangle\right|^2.
\end{align}
Similar to the proof of Theorem \ref{theorem:step-size}, we can expand the matrix exponential:
\begin{align*}
    \nabla_A \left( e_j^\top C  \sum_{k=0}^\infty \frac{t^k}{k!}A^k s \right) = 
    \sum_{k=1}^{\infty} \sum_{r=0}^{k-1} \frac{t^k}{k!}{(A^\top)}^r C^\top e_js^\top {(A^\top)}^{k-1-r}.
\end{align*}
If we choose $\tilde Y = uv^\top$, where $u$ and $v$ are the unit-norm right and left eigenvectors of $A$ corresponding to its eigenvalue $\lambda_2$, we obtain
\begin{align*}
 \left\langle \tilde Y, \nabla_A \left( e_j^\top C \sum_{k=0}^\infty \frac{t^k}{k!}A^k s \right) \right\rangle & =
 \sum_{k=1}^\infty \sum_{r=0}^{k-1} \frac{t^k}{k!} \lambda_2^{k-1} \langle u, C^\top e_j \rangle \langle v, s\rangle \\
 & = \sum_{k=1}^\infty \frac{t^k}{(k-1)!} \lambda_2^{k-1} \langle u, C^\top e_j \rangle \langle v, s\rangle \\
 & = t e^{\lambda_2 t} \langle u, C^\top e_j \rangle \langle v, s\rangle \\
 & = t e^{\lambda_2 t} \langle Cu, e_j \rangle \langle v, s \rangle.
\end{align*}
Remember that (\ref{eqn:upper-bound-1}) is an upper bound for the smallest eigenvalue of $f_1$, and so is
\begin{align*}
\sum_{j=1}^n {\left| \left\langle \tilde Y, \nabla_A(e_j^\top C e^{At}s - e_j^\top x ) \right\rangle \right|}^2 & =  \sum_{j=1}^n t^2 e^{2 \text{Re}(\lambda_2) t} \left|\langle Cu,e_j \rangle \right|^2 \left|\langle v, s\rangle\right|^2 \\
& =  \rho^2 t^2 e^{2 \text{Re}(\lambda_2) t} \left|\langle v, s\rangle\right|^2,
\end{align*}
where $\rho^2 = \|C u\|_2^2$.
If we have multiple trajectories, this upper bound will become
\[ \sum_{k \in \mathcal K} \sum_{t \in \mathcal T_k} \rho^2 t^2 e^{2 \text{Re}(\lambda_2) t} \left|\langle v, s_k\rangle\right|^2, 
\]
where $\{s_k\}_{k \in \mathcal K}$ is the set of initial states of the trajectories. We can bring this upper bound into a form independent of $v$:
\[ \lambda_{\max} \left(\sum_{k \in \mathcal K} \sum_{t \in \mathcal T_k} \rho^2 t^2 e^{2 \text{Re}(\lambda_2) t} s_k s_k^\top \right).
\]

This shows that the ratio of the largest eigenvalue  to the smallest eigenvalue of $f_1$ satisfies
\[ \frac{\lambda_{\max}(f_1)}{\lambda_{\min}(f_1)} \ge \frac{
\lambda_{\min} \left( \rho_1^2 \sum_{k \in \mathcal K} \sum_{t \in \mathcal T_k} t^2 e^{2 \text{\emph{Re}}(\lambda_1) t} \hat s_k \hat s_k^\top \right)
}{
\lambda_{\max} \left( \rho_2^2 \sum_{k \in \mathcal K} \sum_{t \in \mathcal T_k} t^2 e^{2 \text{\emph{Re}}(\lambda_2) t} \hat s_k \hat s_k^\top \right)
}
\]
for any pair of eigenvalues $(\lambda_1, \lambda_2)$ of $\hat A$, where $\rho_1 = \|C u_1\|_2$, $\rho_2 = \|C u_2\|_2$, and $u_1$, $u_2$ are the right eigenvectors of $\hat A$ corresponding to $\lambda_1$, $\lambda_2$. 
If $H$ denotes the Hessian of the loss function $\ell$ at the point $(\hat A, \hat C, \{\hat s_k\}_{k \in \mathcal K})$, we have $\lambda_{\max}(H) \ge \lambda_{\max}(f_1)$ and $\lambda_{\min}(H) \le \lambda_{\min}(f_1)$. Therefore, we also have
\begin{equation} \frac{\lambda_{\max}(H)}{\lambda_{\min}(H)} \ge \frac{
\lambda_{\min} \left( \rho_1^2 \sum_{k \in \mathcal K} \sum_{t \in \mathcal T_k} t^2 e^{2 \text{\emph{Re}}(\lambda_1) t} \hat s_k \hat s_k^\top \right)
}{
\lambda_{\max} \left( \rho_2^2 \sum_{k \in \mathcal K} \sum_{t \in \mathcal T_k} t^2 e^{2 \text{\emph{Re}}(\lambda_2) t} \hat s_k \hat s_k^\top \right)
}.
\label{this:equation}
\end{equation}

To understand the relationship of (\ref{this:equation}) to the convergence rate, consider a quadratic function $h: \mathbb R^n \mapsto \mathbb R$ defined as
\[ h(w) = \frac{1}{2}(w - w^*)^\top H (w - w^*), \]
where $H$ is the Hessian of $h$ and $w^*$ is the point where $h$ attains its minimum. For the gradient descent algorithm
\[ w \gets w - \delta H(w - w^*) \]
to converge to the minimum of $h$ from arbitrary initializations, we need the learning rate $\delta$ to be smaller than $\frac{2}{\lambda_{\max}(H)}$. Assume $(w_0- w^*)$, where $w_0$ is the initial point where the algorithm starts, is in the direction of the eigenvector of $H$ corresponding to its minimum eigenvalue. In other words,
\[ H(w_0 - w^*) = \lambda_{\min}(H) (w_0 - w^*). \]
Then the iterations of the gradient descent algorithm becomes
\begin{align*} (w_k-w^*) & \gets (w_{k-1}-w^*) - \delta H(w_{k-1} - w^*) \\
& \gets (w_{k-1} - w^*) - \delta \lambda_{\min}(H) (w_{k-1}- w^*) \\
& \gets ( 1 - \delta \lambda_{\min}(H) )(w_{k-1} - w^*) \\
& \gets ( 1 -  \delta \lambda_{\min}(H) )^{k}(w_0 - w^*). \end{align*}
Attaining $\|w_k - w^*\|_2 \le \epsilon$ for any $\epsilon > 0$ will require
\[ (1- \delta \lambda_{\min}(H) )^{k}\|w_0 - w^*\|_2 \le \epsilon  \implies k \log(1- \delta \lambda_{\min}(H) ) + \log(\|w_0 - w^*\|_2) \le \log(\epsilon),\]
which gives a lower bound for the number of iterations needed:
\[k \ge \frac{1}{\log\left( \frac{1}{1 - \delta \lambda_{\min}(H)} \right)}\left( \log\left(\frac{1}{\epsilon}\right) + \log(\|w_0 - w^*\|_2) \right).
\]
As a result, convergence of the gradient descent algorithm to the minimum of $h$ in the direction of the bottom eigenvector of $H$ requires
\begin{equation}
\label{that:equation} 
O \left( \left[ \log\left( 
(1 - \delta \lambda_{\min}(H) )^{-1}
\right)
\right]^{-1}
\right) 
\end{equation}
iterations. Remember that for convergence of the algorithm, we require $\delta < \frac{2}{\lambda_{\max}(H)}$; therefore, $\delta \lambda_{\min}(H) < 2 \frac{\lambda_{\min}(H)}{\lambda_{\max}(H)}$. Combining (\ref{this:equation}) and (\ref{that:equation}) gives the local convergence rate for the loss function $\ell$, if we assume the second approximation of $\ell$ represents it well around $(\hat A, \hat C, \{\hat s_k\}_{k \in \mathcal K})$.

\section{Alternatives for Initial States}
\label{app:initial}
For the proof of Theorem \ref{theorem:step-size} and Theorem \ref{second-theorem}, we considered the loss function
\[ 
\ell(A, C, s) = \frac{1}{2} \sum_{t \in \mathcal T} \|x(t) - C e^{At} s\|_2^2,
\]
and analyzed the linearization of the dynamics of the gradient descent algorithm around the solution $(\hat A, \hat C, \hat s)$:
\begin{subequations}
\label{sys:big}
\begin{align}
    \tilde A & \gets f_{1,1}(\tilde A) +  f_{1, 2}(\tilde C) + f_{1,3}(\tilde s) \\
    \tilde C & \gets f_{2,1}(\tilde A) +  f_{2, 2}(\tilde C) + f_{2,3}(\tilde s) \\
    \tilde s & \gets f_{3,1}(\tilde A) +  f_{3, 2}(\tilde C) + f_{3,3}(\tilde s),
\end{align}
\end{subequations}
where $\{f_{i, j}\}_{i \in [3], j\in [3]}$ are the Jacobians of the partial derivatives of $\ell$ with respect to $A$, $C$ and $s$, evaluated at the point $(\hat A, \hat C, \hat s)$. We used the fact that system (\ref{sys:big}) can be represented by a symmetric matrix to use only the eigenvalues of $f_{1,1}$ in order to obtain a lower bound for the largest eigenvalue of the system (\ref{sys:big}).

Note that fixing the initial state $s$ and not updating it with the gradient descent algorithm will not affect the eigenvalues of $f_{1,1}$. Therefore, the results for Theorem \ref{theorem:step-size} and Theorem \ref{second-theorem}, which only depend on the largest eigenvalues of $f_{1, 1}$, will still hold when the initial state is fixed.

Now assume the initial state is obtained via a state estimator:
\[ s = g_{\phi}( \{t, x(t)\}_{t \in \mathcal T}),\]
where $\mathcal T$ is the set of sampling instants for the trajectory and $\{x_t\}_{t \in \mathcal T}$ is the set of samples obtained. While solving the problem
\begin{align*}
    \underset{A, C, \phi}{\text{minimize}} & \quad  \sum_{t \in \mathcal T} \ell \left( x(t), Ce^{At}g_\phi\left( \{t, x(t) \}_{t \in \mathcal T} \right) \right)
    + \mathcal L\left( \phi \right),
\end{align*}
the linear approximation to the gradient descent algorithm can be written as
\begin{subequations}
\label{sys:second-big}
\begin{align}
    \tilde A & \gets \hat f_{1,1}(\tilde A) +  \hat f_{1, 2}(\tilde C) + \hat f_{1,3}(\tilde \phi), \\
    \tilde C & \gets \hat f_{2,1}(\tilde A) +  \hat f_{2, 2}(\tilde C) + \hat f_{2,3}(\tilde \phi), \\
    \tilde \phi & \gets \hat f_{3,1}(\tilde A) +  \hat f_{3, 2}(\tilde C) + \hat f_{3,3}(\tilde \phi),
\end{align}
\end{subequations}
where $\{\hat f_{i, j}\}_{i \in [3], j\in [3]}$ are the Jacobians of the partial derivatives of $\ell$ with respect to $A$, $C$ and $\phi$, evaluated at the point $(\hat A, \hat C, \hat \phi)$.
Note that system (\ref{sys:second-big}) can still be represented by a symmetric matrix; therefore, the largest eigenvalues of $\hat f_{1, 1}$ can be used to obtain an upper bound on the learning rate of the algorithm. Furthermore, given that $\frac{\partial \mathcal L}{\partial A} = 0$, $f_{1, 1}$  in (\ref{sys:big}) and $\hat f_{1, 1}$ in (\ref{sys:second-big}) are identical, with the substitution $s = g_{\phi}\left( \{t, x(t) \}_{t \in \mathcal T} \right)$. For this reason, the results of Theorem \ref{theorem:step-size} and Theorem \ref{second-theorem} still hold for systems with a state estimator $g_{\phi}$, provided that the estimation error at equilibrium is zero; that is, 
\[ \sum_{t \in \mathcal T} \ell\left( x(t), \hat C e^{\hat A t} g_{\hat \phi}\left( \{t, x(t) \}_{t \in \mathcal T} \right) \right) = 0, \]
which is needed only to allow the use of Lemma \ref{first-lemma}.

\section{Additional Experiments}
\label{sec:add-experiment}

In this section, we provide additional experimental results to show that the comparison in Figure~\ref{fig:seed5} is not incidental. Figure \ref{fig:seed10} and Figure \ref{fig:seed20} demonstrate the comparison of the estimated eigenvalues for a different initialization and for a system with a four-dimensional state space, respectively.
% The code used for producing these plots is provided in Supplementary Material. 
Note that we were not able to enable the gradient descent algorithm to learn any of the eigenvalues correctly when the training loss is mean-squared error despite the fact that we used various learning rates for these experiments.

\begin{figure}[ht]
\centering
\subfigure[Eigenvalues with mean-squared-error]{
\includegraphics[trim={0cm 0 0cm 0}, clip, scale=0.65]{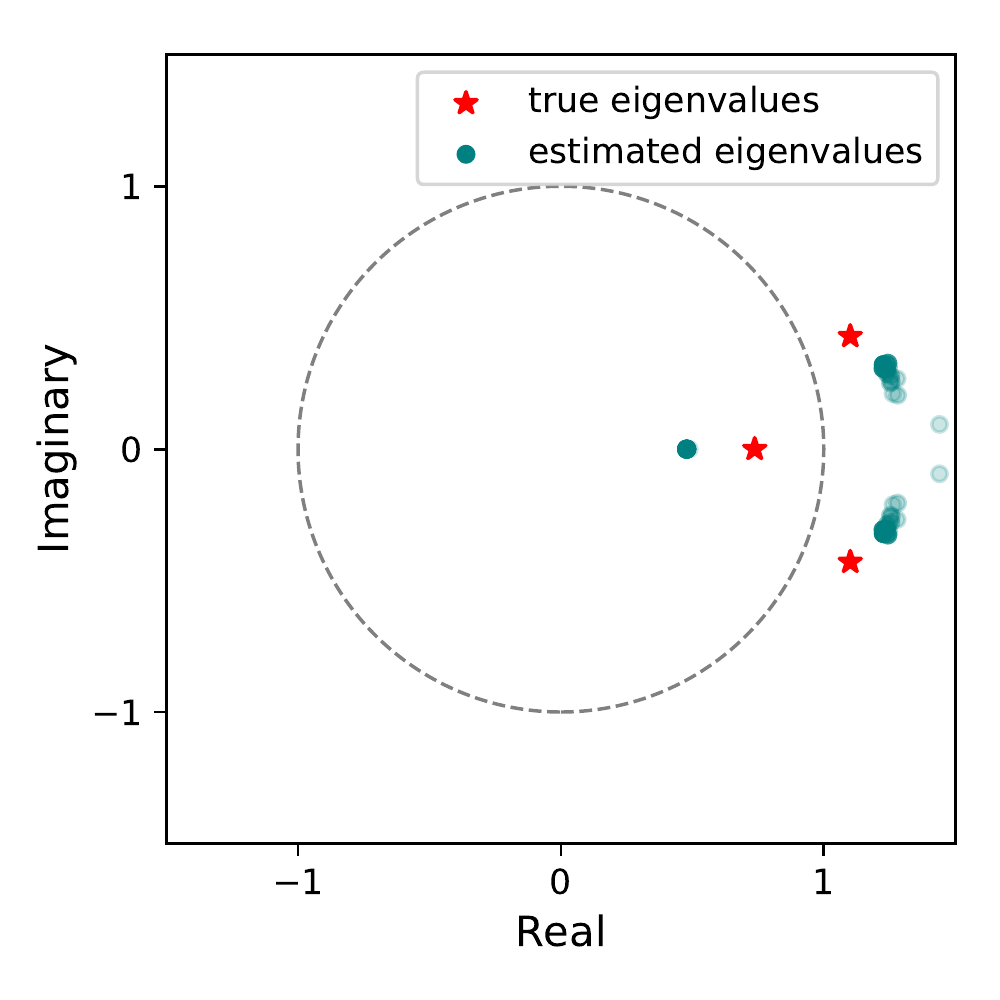}
}
\subfigure[Eigenvalues with logarithmic loss]{
\includegraphics[trim={0cm 0 0cm 0}, clip, scale=0.65]{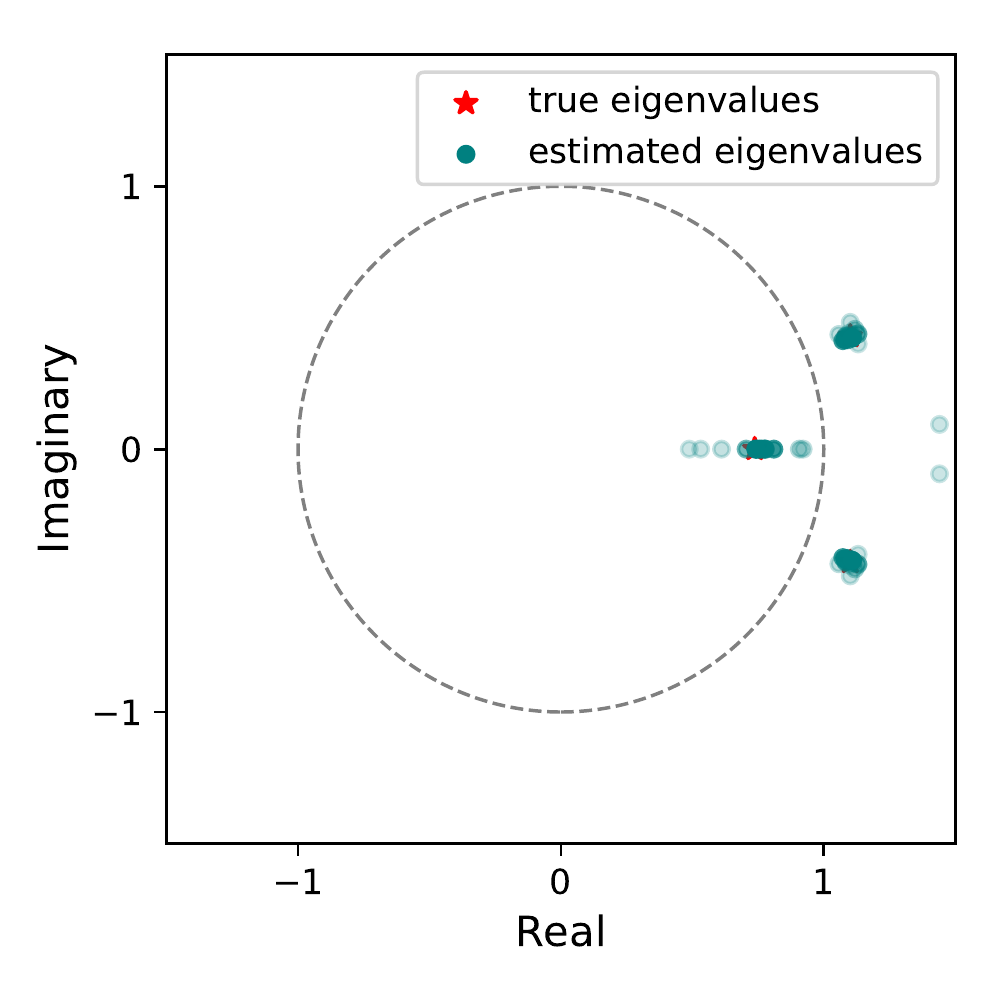}
}
\caption{A linear system with three-dimensional state space is trained with mean-squared-error loss [left] and time-weighted logarithmic loss [right]. The red stars show the eigenvalues of the real system, whereas the green dots show the eigenvalues of the estimated system. Earlier estimates of the eigenvalues are depicted with faded colors.}
\label{fig:seed10}
\end{figure}

\begin{figure}[ht]
\centering
\subfigure[Eigenvalues with mean-squared-error]{
\includegraphics[trim={0cm 0 0cm 0}, clip, scale=0.65]{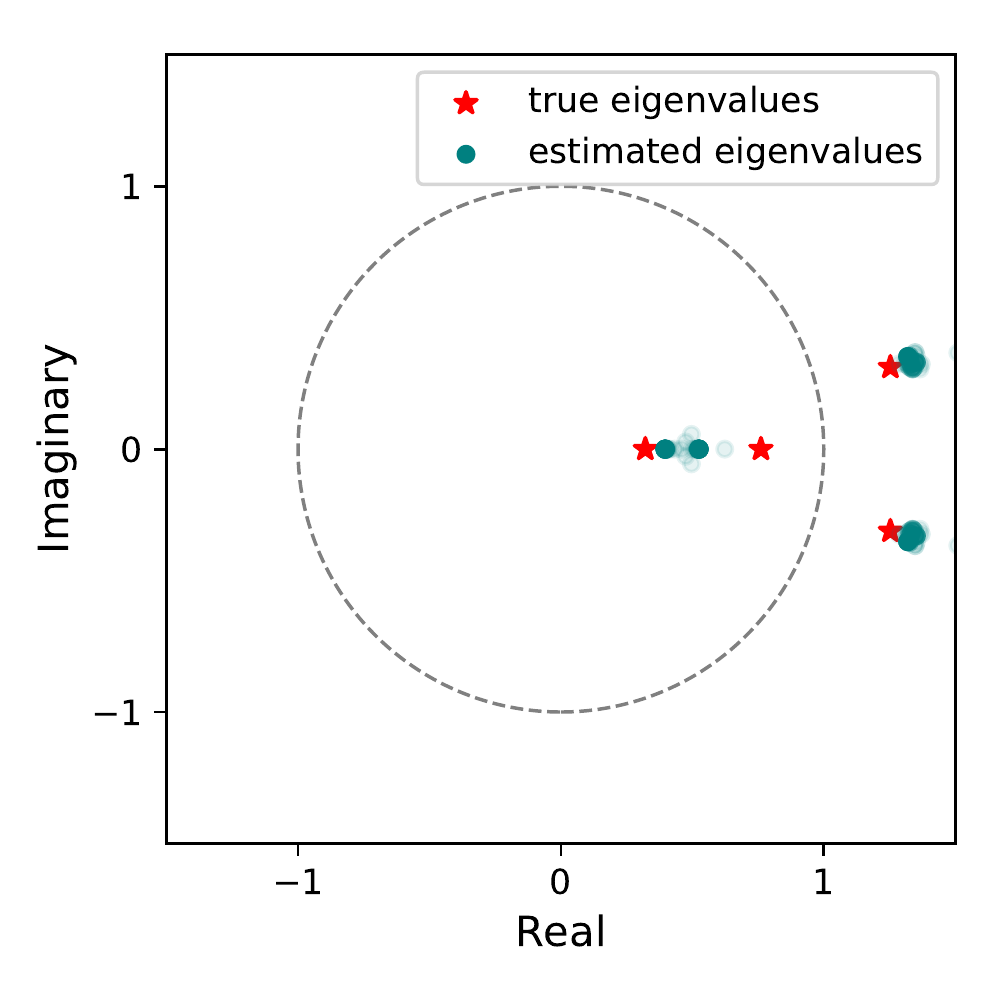}
}
\subfigure[Eigenvalues with logarithmic loss]{
\includegraphics[trim={0cm 0 0cm 0}, clip, scale=0.65]{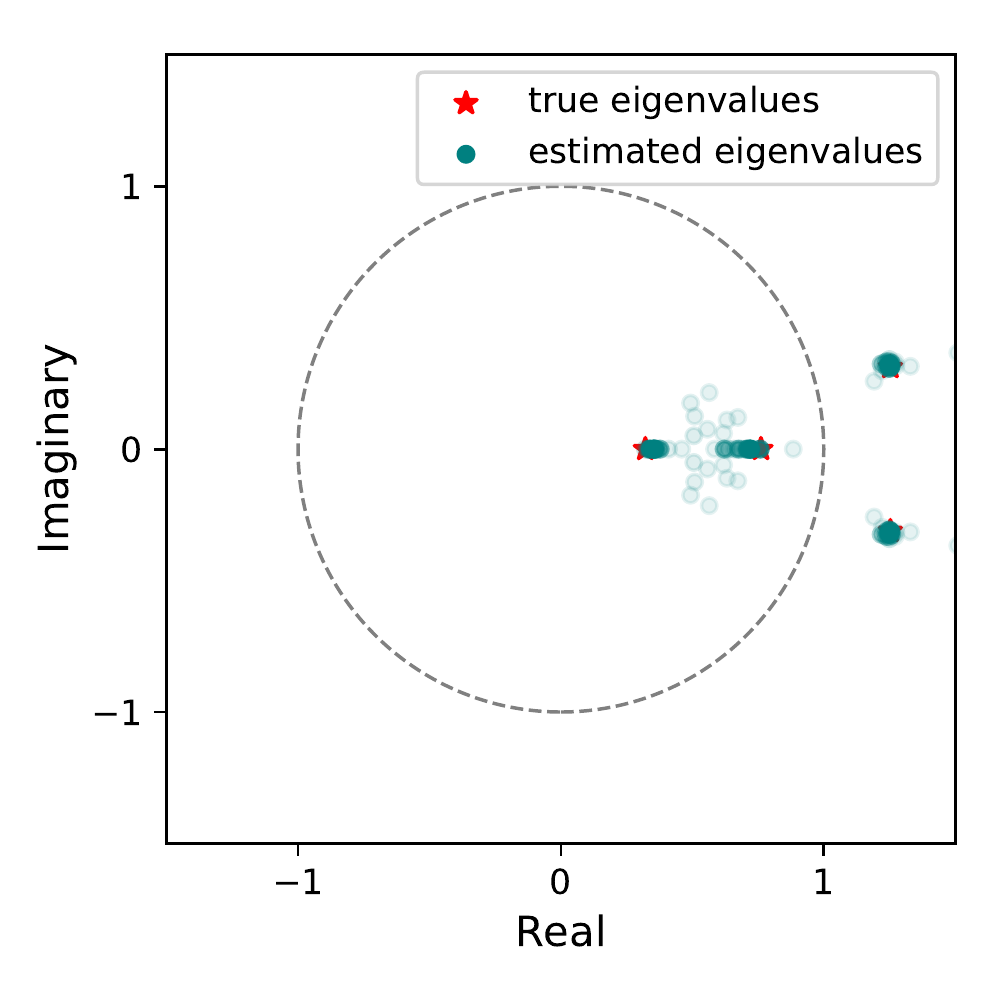}
}
\caption{A linear system with four-dimensional state space is trained with mean-squared-error loss [left] and time-weighted logarithmic loss [right]. The red stars show the eigenvalues of the real system, whereas the green dots show the eigenvalues of the estimated system. Earlier estimates of the eigenvalues are depicted with faded colors.}
\label{fig:seed20}
\end{figure}

\end{document}